\documentclass[conference,onecolumn]{IEEEtran}
\IEEEoverridecommandlockouts
\usepackage{cite}
\usepackage{amsmath,amssymb,amsthm,booktabs,graphicx,microtype,cite}
\usepackage{algorithm}
\usepackage{algpseudocode}
\usepackage{graphicx}
\usepackage{textcomp}
\usepackage{balance}
\usepackage{xcolor}
\usepackage{url}
\def\BibTeX{{\rm B\kern-.05em{\sc i\kern-.025em b}\kern-.08em
    T\kern-.1667em\lower.7ex\hbox{E}\kern-.125emX}}
\newtheorem{proposition}{Proposition}

\newcommand{\norm}[1]{\left\lVert #1\right\rVert}
\newcommand{\R}{\mathbb{R}}

\begin{document}

\title{
Beyond Conventional Federated Learning \\via High-Order Regularization 
\thanks{The Research Foundation Flanders (FWO) research project G081222N and UA BOF DocPRO4 projects with ID 46929 and 48996 partially supported the paper's authors.}
}

\author{
	\IEEEauthorblockN{Alireza Kabgani, Masoud Ahookhosh}
\IEEEauthorblockA{\textit{Department of Mathematics}, 
\textit{University of Antwerp}, 
Antwerp, Belgium \\
Email: alireza.kabgani@uantwerp.be, masoud.ahookhosh@uantwerp.be}
}

\maketitle
\pagestyle{plain}
\thispagestyle{plain}
\begin{abstract}
Federated clients that perform several local optimization steps can return parameter displacements with widely different magnitudes. The quadratic regularization of FedProx grows linearly with displacement and therefore offers limited control over the contrast between ordinary and unusually large client movements. We here introduce HiFedProx, which replaces the quadratic penalty with a scale-matched power-type regularizer indexed by $p\geq2$. All powers have the same regularization-gradient magnitude at a reference displacement $R$, while every $p>2$ gives a weaker response below $R$ and a stronger response above it. An exact affine reference calculation shows that increasing $p$ compresses relative displacement disparities, although very large powers approach fixed-radius behavior and increase local curvature. HiFedProx combines this geometry with finite-budget stochastic client optimization and same-minibatch Armijo backtracking. In paired five-seed experiments on a frozen 60-writer FEMNIST subset, a common-parameter study over $p\in\{2,3,4,5,6,7,8\}$ shows similar clean-training performance but substantial gains under composite stress. The lowest moderate- and severe-stress losses occur at $p=7$ and $p=6$, improving over $p=2$ by $11.44\%$ and $23.16\%$, respectively. Although displacement-tail ratios continue to decrease through $p=8$, predictive performance peaks in an intermediate range and Armijo trial cost increases with $p$. These results indicate that the exponent should be calibrated rather than maximized. In our experiments, $p=5$--$7$ provides the most useful range.
\end{abstract}


\section{Introduction}
Federated learning (FL) trains a model by solving a distributed optimization problem in which multiple clients collaboratively train a shared model from distributed data without transmitting their raw examples to a central server. We consider a system consisting of one central server and $N \geq 1$ clients, where client $i$ stores a local dataset 
$\mathcal{D}_i=(z_{ij})_{j=1}^{n_i}$ and minimizes the empirical objective $f_i$ given by
\begin{equation} \label{eq:local-objective}
    f_i(w)=\frac{1}{n_i}\sum_{j=1}^{n_i}\ell(w;z_{ij}).
\end{equation}
The learning task is to solve the client-uniform optimization problem
\begin{equation}\label{eq:global-objective}
    \min_{w\in\mathbb{R}^d} F(w)=\frac{1}{N}\sum_{i=1}^{N}f_i(w),
\end{equation}
where each client contributes equally to the global objective regardless of its local sample size. 

A widely used approach is Federated Averaging (FedAvg), in which the server broadcasts the current model to a subset of clients, each client performs several local stochastic-gradient updates, and the resulting local models are averaged to obtain the next global iterate~\cite{mcmahan2017fedavg}. While multiple local updates substantially reduce communication, they also amplify the effects of statistical heterogeneity and optimization variability, often leading to highly unequal client displacements. In particular, some clients may consistently move much farther from the global model because of heterogeneous data distributions, aggressive local optimization, or larger local workloads. This imbalance motivates optimization methods that explicitly account for disparities in local update magnitudes while preserving the communication efficiency of federated learning.

To moderate excessive client displacements, FedProx augments each client's local optimization problem with a quadratic proximal term~\cite{li2020fedprox}, i.e.,
\begin{equation}\label{eq:fedProx}
    \min_{u\in\mathbb{R}^d}~ Q_i(u;w^t):=f_i(u) + \frac{\mu}{2} \|u-w^t\|^2,
\end{equation}
where $w^t$ denotes the current global model. The quadratic regularization penalizes deviations from $w^t$, and increasing the parameter $\mu$ uniformly strengthens this penalty for all client displacements. We instead ask whether the regularizer can remain relatively mild for typical client displacements while responding more aggressively to unusually large ones.

Motivated by the need for a more flexible regularization of client displacements, we generalize the quadratic proximal term to a family of high-order regularizers inspired by recent advances in high-order proximal-point methods~\cite{nesterov2021inexact,Nesterov2023a, Ahookhosh24,Kabganidiff,KabganiItsDEAL,Kabgani24itsopt,Ahookhosh2025,Kabgani2026Robust}. Specifically, we consider the local optimization models
\begin{equation}\label{eq:intro-model}
    \min_{u\in\R^d} \quad Q_{i,p}(u;w^t):=f_i(u)+\frac{\mu}{pR^{p-2}}\|u-w^t\|^{p},
\end{equation}
where $ p\ge2$ and $\mu>0$ controls the regularization strength, $R>0$ is a reference displacement scale, and $p$ provides an additional degree of flexibility. The normalization is chosen so that the magnitude of the regularization gradient equals $\mu R$ whenever $\|u-w^t\|=R$, independently of $p$. Consequently, compared with the scale-matched quadratic regularizer (i.e., $p=2$), every choice $p>2$ yields a milder regularization response for displacements below $R$ and a stronger response for displacements above $R$. 
We restrict attention to $p\geq2$ because the subquadratic regime $1<p<2$ reverses
this behavior, i.e., it penalizes small displacements more strongly and large displacements more weakly than the quadratic model.

We incorporate the family \eqref{eq:intro-model} into a finite-budget stochastic client solver with same-minibatch Armijo backtracking and call the resulting method HiFedProx. All choices of $p\geq 2$ use the same client initialization, sampling rule, local optimization procedure, server aggregation, and stochastic-gradient budget and only the radial regularization term
changes. The method remains first order and requires neither Hessians nor higher-order derivatives. The case $p=2$ gives the safeguarded quadratic comparator, while values $p>2$
produce high-order variants.

The paper makes three main contributions. \textit{First}, we introduce a scale-matched family of high-order proximal regularizers for federated learning and 
show how the parameters $\mu$, $R$, and $p$ separately control the strength, crossover scale, and shape of the regularization response.
 \textit{Second}, we develop HiFedProx, a complete first-order federated procedure with finite local budgets and capped same-minibatch Armijo backtracking. We also derive an exact affine reference response and a smooth fixed-batch safeguard result that clarify the radial mechanism and the role of backtracking. 
 \textit{Third}, we conduct a paired FEMNIST exponent-sensitivity study over 
 $p\in\{2,3,4,5,6,7,8\}$ with matched gradient and communication budgets. The results identify a useful intermediate range, around $p=5$--$7$ in the present protocol: larger powers continue to compress displacement disparities, but predictive performance saturates and then deteriorates while line-search cost increases.
 
The remainder of this paper is organized as follows. In Section~\ref{sec:related}, we discuss the related works. In Section~\ref{sec:method}, we introduce scale-matched regularization and HiFedProx. In Section~\ref{sec:geometry}, the geometry and local safeguard are discussed. 
Section~\ref{sec:experiment} presents the experimental design, Section~\ref{sec:results} reports the FEMNIST results, and Section~\ref{sec:discussion} discusses the choice of $p$ and the limitations. Section~\ref{sec:conclusion} concludes the paper.

%%%%%%%%%%%%%%%%%%%%%%%%%%%%%%%%%%%%%%%%%%%%%%%%%%%%%%%%%%%%%%%%%%%%%%%%%%%%%%%%%%%%
%%%%%%%%%%%%%%%%%%%%%%%%%%%%%%%%%%%%%%%%%%%%%%%%%%%%%%%%%%%%%%%%%%%%%%%%%%%%%%%%%%%%
\section{Related Work} \label{sec:related}
FedProx is the closest related method, as it augments each client's local objective with a quadratic proximal regularizer while maintaining a single shared global model~\cite{li2020fedprox}. Similarly, pFedMe employs quadratic proximal subproblems, but uses them to learn personalized client models rather than a single global model~\cite{dinh2020pfedme}. More recently, adaptive proximal methods have focused on dynamically adjusting the coefficient of a quadratic regularizer while retaining its quadratic form~\cite{hajarizadeh2026dynamu}. Our work differs in that it generalizes the quadratic proximal regularizer to a scale-matched family of high-order regularizers. Rather than varying only the regularization coefficient, we introduce the exponent $p$ as an additional design parameter, with the regularizers calibrated to have the same gradient magnitude at a prescribed reference displacement $R$.

Several federated optimization methods address sources of performance degradation other than the choice of proximal regularizer. SCAFFOLD, FedNova, and FedVARP mitigate client drift, heterogeneous local computation, and variance arising from partial client participation, respectively~\cite{karimireddy2020scaffold,wang2020fednova,jhunjhunwala2022fedvarp}. These approaches are complementary to the high-order proximal regularization considered in this work. In a different direction, nonsmooth $\ell_1$-based regularizers promote sparse, coordinatewise client updates~\cite{shi2023fedl1}, whereas our regularization is smooth, isotropic, and based on high-order Euclidean norms.

Federated second-order methods, such as FedNL, exploit Hessian information to construct cubic Taylor models for local optimization~\cite{safaryan2022fednl}. In contrast, our approach modifies the proximal regularization in the client objective while remaining entirely first-order. Furthermore, server-side clipping limits client updates only after local training has been completed, whereas our high-order proximal regularization influences the optimization trajectory throughout local training; see, for example, analyses of clipped FedAvg~\cite{zhang2022clipping}.

%%%%%%%%%%%%%%%%%%%%%%%%%%%%%%%%%%%%%%%%%%%%%%%%%%%%
\section{Scale-Matched Regularization and HiFedProx}
\label{sec:method}
\subsection{Federated objective and scale matching}
In this section, we consider the client-uniform optimization problem \eqref{eq:global-objective}, in which each client contributes equally to the global objective regardless of its local sample size. The client sampling and model aggregation procedures follow the same client-uniform convention. The high-order regularizer in \eqref{eq:intro-model} is applied only during local training; predictive loss and accuracy are evaluated using the resulting model without the regularization term.

Let us define the power regularizer and the radial map
\[
 \phi_{p,R,\mu}(d)=\frac{\mu}{pR^{p-2}}\norm{d^p},
 \quad
 J_p(d)=\begin{cases}\norm{d^{p-2}}d,&d\ne0,\\0,&d=0,\end{cases}
\]
where its gradient and magnitude are given by
\begin{align}
 \nabla\phi_{p,R,\mu}(d)&=\mu R^{2-p}J_p(d),\label{eq:reg-gradient}\\
 \norm{\nabla\phi_{p,R,\mu}(d)}&=\mu R\left(\frac{\norm d}{R}\right)^{p-1}.\label{eq:reg-magnitude}
\end{align}
All powers therefore have regularization-gradient magnitude $\mu R$ at $\norm d=R$. Comparing a general power with the quadratic response gives
\begin{equation} \label{eq:relative-force}
 \frac{\norm{\nabla\phi_{p,R,\mu}(d)}}{\norm{\nabla\phi_{2,R,\mu}(d)}}
 =\left(\frac{\norm d}{R}\right)^{p-2}.
\end{equation}
For $p=3$, the magnitude is $\mu\norm d^2/R$, compared with $\mu\norm d$ for $p=2$. Figure~\ref{fig:scale-match} shows the normalized responses. The parameter $R$ is the crossover at which they agree, $\mu$ rescales the response, and $p$ controls its shape. Only the gradient magnitude is matched at $R$, and the regularizer values and curvatures remain power dependent.
The parameter $R$ is a crossover scale, not a clipping radius or hard constraint.

\begin{figure}[h]
 \centering
 \includegraphics[width=0.55\textwidth]{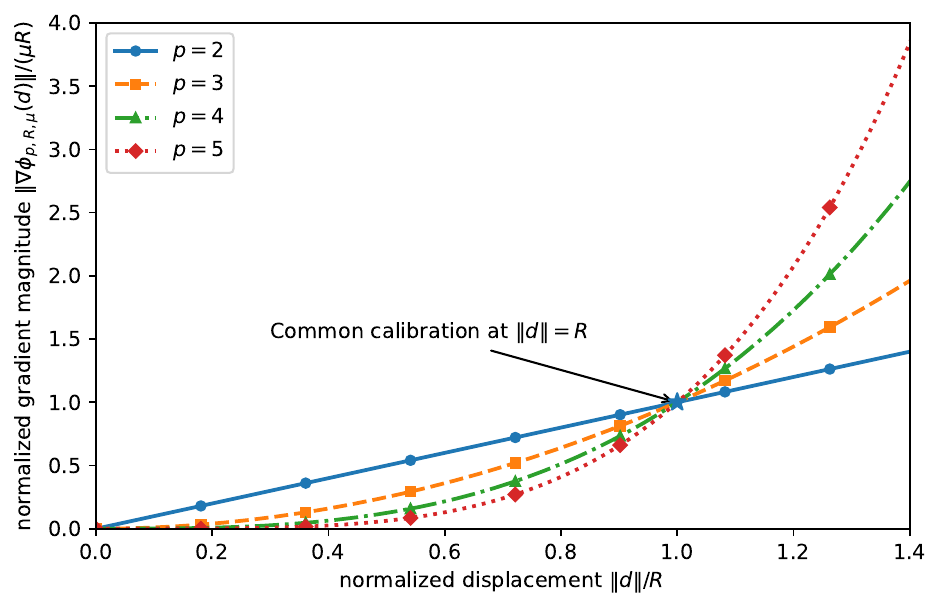}
 \caption{Normalized quadratic and $p$th-order regularization-gradient magnitudes. Both equal one at $\norm d/R=1$.}
 \label{fig:scale-match}
\end{figure}

\subsection{Finite-budget client update}
Training proceeds in communication rounds. 
At round $t$, the server has the current model $w^t$ and samples a subset $S_t\subseteq\{1,\ldots,N\}$ of $m:=|S_t|$ clients uniformly without replacement. The server sends $w^t$ to the selected clients.
Each selected client starts from $w^t$, returns a finite-budget point $u_i^t$, and communicates the displacement $\Delta_i^t=u_i^t-w^t$. The equal-weight server update is
\begin{equation}\label{eq:server-update}
 w^{t+1}=w^t+\frac1{|S_t|}\sum_{i\in S_t}\Delta_i^t
 =\frac1{|S_t|}\sum_{i\in S_t}u_i^t.
\end{equation}
For a nonempty minibatch $B\subseteq\mathcal D_i$, let
\[
 Q_{i,p}(u;w^t,B)=f_i(u;B)+\phi_{p,R,\mu}(u-w^t).
\]
Let $\operatorname{AD}_u Q_{i,p}(u;w^t,B)$ denote the direction returned by automatic differentiation for the complete regularized batch model. At differentiable points, this is $\nabla_u Q_{i,p}(u;w^t,B)$. Algorithm~\ref{alg:client} uses this direction and evaluates the current value and every trial value on the same minibatch. If no trial passes the test, the client keeps its current iterate.

\begin{algorithm}[h]
\caption{HiFedProx (High-order Fed-Prox)}
\label{alg:client}
\small
\begin{algorithmic}[1]
\Require $w^t,\mathcal D_i,p,\mu,R,K_i^t,\alpha_0,c,\rho,J$
\State $u\gets w^t$
\For{$k=0,\ldots,K_i^t-1$}
 \State sample nonempty $B_k\subseteq\mathcal D_i$
 \State $g\gets\operatorname{AD}_u Q_{i,p}(u;w^t,B_k)$
 \State $q\gets Q_{i,p}(u;w^t,B_k)$; $\alpha\gets\alpha_0$
 \For{$j=0,\ldots,J$}
  \State $v\gets u-\alpha g$
  \If{$Q_{i,p}(v;w^t,B_k)\le q-c\alpha\norm g^2$}
   \State $u\gets v$; \textbf{break}
  \EndIf
  \If{$j<J$} \State $\alpha\gets\rho\alpha$ \EndIf
 \EndFor
\EndFor
\State \Return $u-w^t$
\end{algorithmic}
\end{algorithm}

Note that each local iteration uses one backward pass. The current objective value is obtained from the same forward computation, and each Armijo trial adds one further forward evaluation. The backtracking procedure therefore adds neither backward passes nor communication.
Moreover, let us emphasize that setting $p=2$ gives the safeguarded quadratic comparator. 
Every value $p\geq2$ uses the same client optimizer, sampling rule, aggregation rule, and gradient budget, and only the radial power changes. No variant uses Hessians or higher-order derivatives.

%%%%%%%%%%%%%%%%%%%%%%%%%%%%%%%%%%%%%%%%%%%%%%%%%%%%%%%%%%%%%%%%%%%%%%%%%%%%%%%%%%%%
%%%%%%%%%%%%%%%%%%%%%%%%%%%%%%%%%%%%%%%%%%%%%%%%%%%%%%%%%%%%%%%%%%%%%%%%%%%%%%%%%%%%
\section{Geometry and Local Safeguard} \label{sec:geometry}
The following calculation isolates the radial effect before curvature, stochastic sampling, and repeated local steps enter. It treats $g\in\R^d$ as a prescribed vector. When the loss is differentiable, $g$ can be chosen as a client gradient.

\begin{proposition}[Affine client-model response]
\label{prop:linearized}
Let $p\ge2$, $\mu,R>0$, and $g\in\R^d$. The problem
\begin{equation}
 \min_d\left\{\langle g,d\rangle+\frac{\mu}{pR^{p-2}}\norm d^p\right\}
 \label{eq:linearized-model}
\end{equation}
has the unique solution $d_p(0)=0$ and, for $g\ne0$,
\begin{equation}
 d_p(g)=-R\left(\frac{\norm g}{\mu R}\right)^{1/(p-1)}\frac{g}{\norm g}.
 \label{eq:linearized-step}
\end{equation}
For nonzero $g_i,g_j$,
\begin{equation}
 \frac{\norm{d_p(g_i)}}{\norm{d_p(g_j)}} =\left(\frac{\norm{g_i}}{\norm{g_j}}\right)^{1/(p-1)}.
 \label{eq:relative-spread}
\end{equation}
\end{proposition}
\emph{Proof.} The objective is coercive and strictly convex. For $g\ne0$, the optimality condition
$g+\mu R^{2-p}\norm d^{p-2}d=0$
shows that $d$ points opposite to $g$. Taking norms and restoring the direction gives~\eqref{eq:linearized-step} and division gives~\eqref{eq:relative-spread}.
\hfill$\square$

Writing $\theta(g)=\norm g/(\mu R)$ gives the dimensionless response
\begin{equation}
 \frac{\norm{d_p(g)}}{R}=\theta(g)^{1/(p-1)}.
 \label{eq:dimensionless-response}
\end{equation}
The value $\theta=1$ gives the common displacement $R$. For every $p>2$, the high-order displacement is larger than the quadratic one when $0<\theta<1$, equal to it when $\theta=1$, and smaller when $\theta>1$. Moreover, if $\norm{g_i}>\norm{g_j}$, then~\eqref{eq:relative-spread} is strictly smaller than $\norm{g_i}/\norm{g_j}$, and this compression becomes stronger as $p$ increases. The transformation preserves ordering but does not uniformly shrink all displacements.
If $s_{(1)}\leq\cdots\leq s_{(M)}$ are ordered direction norms, the corresponding affine displacement norms satisfy
\begin{equation} \label{eq:ordered-response}
 r_{(k)}^{(p)} =R\left(\frac{s_{(k)}}{\mu R}\right)^{1/(p-1)}.
\end{equation}
This identity motivates the upper-tail displacement diagnostic used in the experiments.

Larger powers strengthen disparity compression, but the limiting behavior warns against the rule ``larger is always better.'' For every fixed $g\neq0$,
\begin{equation} \label{eq:large-p-step}
 \lim_{p\to\infty}\norm{d_p(g)}=R.
\end{equation}
Thus, the affine response approaches a fixed-radius step and becomes nearly insensitive to the magnitude of $g$. Equivalently,
\begin{equation} \label{eq:large-p-regularizer}
 \phi_{p,R,\mu}(d) =\frac{\mu R^2}{p}\left(\frac{\norm d}{R}\right)^p
 \to
 \begin{cases}
 0,&\norm d\leq R,\\
 +\infty,&\norm d>R,
 \end{cases}
\end{equation}
pointwise. Very large powers therefore approximate a hard displacement-radius constraint: they can suppress extreme movements, but may also erase useful magnitude information.

The $p$th-order regularizer also becomes more curved farther away from the server model. On $\{d:\norm d\le D\}$, its gradient is Lipschitz with constant
\begin{equation}
 L_\phi(D)=\mu(p-1)\left(\frac{D}{R}\right)^{p-2}.
 \label{eq:Lphi}
\end{equation}
% Thus, $L_\phi(D)=\mu$ for $p=2$, while $L_\phi(D)=2\mu D/R$ for $p=3$. This motivates adapting the trial step as the client moves away from $w^t$.
For $D>R$, this quantity grows rapidly with $p$, i.e., the larger powers can require smaller accepted steps and more backtracking. For any fixed $D<R$, $L_\phi(D)\to0$ as $p\to\infty$, although the curvature need not decrease monotonically for moderate powers.

For a smooth fixed-batch objective, let $g=\nabla Q(u)$ and suppose that $\nabla Q$ is $L_Q$-Lipschitz on the trial segment. The descent lemma ensures that
\begin{equation}
 Q(u-\alpha g)\le Q(u)-\alpha\left(1-\frac{L_Q\alpha}{2}\right)\norm g^2.
 \label{eq:smooth-armijo-reference}
\end{equation}
Hence, every $\alpha\le2(1-c)/L_Q$, passes the Armijo test with parameter $c$~\cite{armijo1966minimization}. Algorithm~\ref{alg:client} employs this acceptance rule as a finite-budget numerical safeguard.

These formulas motivate the two displacement diagnostics used in the following. Groupwise means show whether the response is selective, and an upper-tail-to-median ratio measures relative spread among the clients selected in each round.

The affine formula describes a reference subproblem rather than a complete federated convergence result. Even with exact client solutions and full participation, equal averaging for $p>2$ is not generally a common-step gradient method on the average high-order envelope: the factor relating a client displacement to its envelope gradient depends on that displacement norm. Finite local budgets and partial participation introduce further errors. We therefore use the analysis to explain the radial mechanism and the role of backtracking, while leaving the outer convergence to future work.

%%%%%%%%%%%%%%%%%%%%%%%%%%%%%%%%%%%%%%%%%%%%%%%%%%%%%%%%%%%%%%%%%%%%%%%%%%%%%%%%%%%%
%%%%%%%%%%%%%%%%%%%%%%%%%%%%%%%%%%%%%%%%%%%%%%%%%%%%%%%%%%%%%%%%%%%%%%%%%%%%%%%%%%%%
\section{Experimental Design}
\label{sec:experiment}

\subsection{Data, model, and stress conditions}
\label{subsec:data-model-stress}

The primary study uses frozen FEMNIST caches constructed from 60 writers, with each writer representing a single client. The caches comprise 8,064 training, 995 calibration, and 1,021 evaluation examples. The pooled training, calibration, and evaluation sets each contain all 62 FEMNIST character classes~\cite{caldas2018leaf}. Although the calibration and evaluation sets are disjoint from the training set, they are drawn from the same writers. Consequently, the experiments evaluate generalization to new examples from known clients rather than to previously unseen clients. Each image is stored as a flattened $14\times14$ grayscale array with pixel values in $[0,1]$ and reshaped before being passed to the network.

The classifier has two $3\times 3$ convolutional layers with unit padding and 16 and 32 output channels, respectively. Each convolution is followed by ReLU and $2\times 2$ max pooling. The resulting $32\times 3\times 3$ representation is flattened and passes through a 64-unit fully connected ReLU layer and a 62-logit output layer. The network has 27,326 trainable parameters. Predictive performance is evaluated
using cross-entropy loss and classification accuracy. The client regularizer is used only during local training.

The training objective is nonconvex but smooth within each fixed activation and pooling pattern, with nondifferentiable boundaries introduced by ReLU and max pooling. At each local step, PyTorch automatic differentiation returns a direction for the complete regularized minibatch objective, following its documented rules in nondifferentiable elementary operations~\cite{pytorchAutograd}. At differentiable points, this direction is the ordinary gradient. Proposition~\ref{prop:linearized} applies to any prescribed direction and therefore provides a reference calculation for the direction returned by automatic differentiation. It does not describe the complete nonlinear client trajectory.
The inequality~\eqref{eq:smooth-armijo-reference} provides us with a smooth fixed-batch motivation for backtracking, while the experiment uses the capped same-minibatch rule as a numerical safeguard at every local step.

\begin{table}[h]
\caption{Primary FEMNIST protocol. In the stressed conditions,
the multipliers apply only to 12 clients selected once per seed
and held fixed throughout training.}
\label{tab:protocol}
\centering
%\footnotesize
\setlength{\tabcolsep}{3.5pt}
\begin{tabular}{@{}ll@{}}
\toprule
Item & Setting \\
\midrule
Training / calibration / evaluation & $8064/995/1021$ \\
Writers / pooled classes & $60/62$ \\
Rounds / clients per round & $100/18$ \\
Minibatch / ordinary local steps & $16/3$ \\
Initial local step / Armijo trial & $0.5$ \\
Clean & labels unchanged, step $\times1$, \\
      & local steps $\times1$ \\
Moderate stress  & cyclic label corruption, step $\times2$, \\
                 & local steps $\times3$ \\
Severe stress  & cyclic label corruption, step $\times4$, \\
               &  local steps $\times3$ \\
Powers in sensitivity study & $p\in \{2,3,4,5,6,7,8\}$ \\
Evaluation seeds & $3001$--$3005$ \\
\bottomrule
\end{tabular}
\end{table}

In each communication round, 18 of the 60 clients are sampled uniformly without replacement. Ordinary selected clients use three local iterations with initial step-size $0.5$. In each stressed condition, 12 clients are selected once for each seed and remain in the stress-test group throughout all 100 rounds. For a fixed seed, this group is shared among the compared methods.
For every stress-test client, all training labels are replaced according to $y\mapsto(y+1)\bmod 62$,
while the calibration and evaluation labels remain unchanged. Under moderate stress, these clients use nine local iterations and an initial step-size $1.0$. Under severe stress, they use nine local iterations and initial step-size $2.0$. The remaining clients retain the normal settings in Table~\ref{tab:protocol}. This composite intervention is designed to induce a persistent group of comparatively large client displacements. The resulting displacement separation is measured in Section~\ref{sec:results} rather than assumed.

\subsection{Exponent protocol, pairing, and evaluation}
\label{subsec:calibration-evaluation}

The reference pair $(\mu,R)=(0.03,1)$ was selected by the original 60-round clean-plus-moderate calibration for the quadratic and cubic variants. We keep this pair fixed for 
\[p\in\{2,3,4,5,6,7,8\},\] allowing the sensitivity study to isolate the effect of the exponent.  Although any real exponent $p\geq2$ is theoretically admissible, we restrict our numerical study to the consecutive integers for a transparent and systematic comparison. Notice that the real-valued powers $p\geq 2$ may further refine the choice of $p$ and are left for future investigation.

For each condition and seed, all powers share the initial network, stress-test clients, participating clients in every round, and client-local minibatch streams. Their local-step, backward-pass, and communication budgets are therefore identical. The Armijo parameters are $c=10^{-4}$, $\rho=0.5$, and $J=12$, so each local search tests at most 13 step sizes. The direction, current value, and all trial values within one local search use the same minibatch. If all trials fail, the beginning-of-step iterate is restored.

The primary outcomes are pooled cross-entropy and pooled accuracy on the 1,021 evaluation examples. We also report the linearly interpolated tenth percentile of the 60 writer accuracies. For the $m=18$ selected displacement norms in round $t$, ordered as $r_{(1)}^t\leq\cdots\leq r_{(m)}^t$, define
\begin{equation}\label{eq:tail-metric}
 T_t=\frac{r_{(\lceil0.9m\rceil)}^t}{\max\{r_{(\lceil0.5m\rceil)}^t,10^{-12}\}}.
\end{equation}
The reported tail statistic averages $T_t$ over rounds. For $m=18$, the numerator and denominator are the 17th and 9th ordered norms. Larger values indicate a more dominant upper displacement tail.

For selected paired comparisons, we enumerate all $5^5=3125$ ordered bootstrap resamples of the five seedwise loss differences and report the $2.5\%$ and $97.5\%$ percentiles of the resampled means~\cite{efron1979bootstrap}. Because the exponent study is exploratory, involving many possible pairwise contrasts, we emphasize effect sizes, intervals, and paired directions rather than treating every comparison as a separate hypothesis test.

\section{Primary FEMNIST Results}
\label{sec:results}

\subsection{Predictive performance and exponent sensitivity}
\label{subsec:predictive-performance}

Table~\ref{tab:main} reports the common-parameter exponent study. Clean losses remain close among all powers, ranging from $2.2610$ to $2.2922$. Under moderate stress, loss decreases from $2.7208$ at $p=2$ to its minimum $2.4095$ at $p=7$, then rises to $2.4328$ at $p=8$. Under severe stress, loss improves through $p=6$, reaching $2.6235$, and then increases to $2.6772$ and $2.6701$ at $p=7$ and $p=8$. The accuracy results show the same broad intermediate-power advantage, although their exact optima differ: $p=5$ gives the highest moderate-stress accuracy, while $p=6$ gives the highest severe-stress accuracy.

\begin{table*}[h]
\caption{Final FEMNIST results over five paired seeds, reported as mean $\pm$ sample standard deviation. All powers use the common setting $(\mu,R)=(0.03,1)$.}
\label{tab:main}
\centering
\footnotesize
\setlength{\tabcolsep}{5.0pt}
\begin{tabular}{cccccc}
\toprule
$p$ & Clean loss & Moderate loss & Severe loss & Moderate acc. (\%) & Severe acc. (\%) \\
\midrule
$2$ & $2.2833\pm0.1000$ & $2.7208\pm0.1087$ & $3.4141\pm0.0913$ & $38.28\pm5.87$ & $11.54\pm4.89$ \\
$3$ & $2.2922\pm0.0943$ & $2.5770\pm0.1722$ & $2.9552\pm0.1055$ & $41.27\pm7.61$ & $30.75\pm5.87$ \\
$4$ & $\mathbf{2.2610\pm0.0637}$ & $2.4928\pm0.1066$ & $2.7991\pm0.1144$ & $43.72\pm3.89$ & $36.06\pm4.01$ \\
$5$ & $2.2690\pm0.0579$ & $2.4188\pm0.1005$ & $2.7059\pm0.1022$ & $\mathbf{45.70\pm3.05}$ & $39.00\pm2.24$ \\
$6$ & $2.2640\pm0.0707$ & $2.4218\pm0.0985$ & $\mathbf{2.6235\pm0.0561}$ & $45.39\pm1.99$ & $\mathbf{40.88\pm1.29}$ \\
$7$ & $2.2634\pm0.0414$ & $\mathbf{2.4095\pm0.0851}$ & $2.6772\pm0.0790$ & $45.41\pm2.60$ & $39.84\pm2.74$ \\
$8$ & $2.2772\pm0.0780$ & $2.4328\pm0.0540$ & $2.6701\pm0.0443$ & $44.98\pm2.43$ & $40.20\pm1.02$ \\
\bottomrule
\end{tabular}
\end{table*}

Figure~\ref{fig:loss-power} makes the intermediate-power plateau visible. On the clean-plus-moderate calibration-style score, $p=7$ has the lowest mean value, $2.3365$, while $p=5$ and $p=6$ remain close at $2.3439$ and $2.3429$. The moderate-stress advantage of $p=7$ over $p=5$ and $p=6$ is small, and the corresponding paired bootstrap intervals contain zero. Under severe stress, $p=6$ improves over $p=5$ by $0.0825$ on average; the exhaustive paired bootstrap interval is $[0.0221,0.1210]$, and four of five seeds favor $p=6$. Its mean severe-stress loss is also below those of $p=7$ and $p=8$, although those pairwise intervals contain zero. The robust conclusion is therefore an intermediate-power plateau rather than a uniquely optimal exponent: additional tail compression beyond this range does not translate into uniformly better prediction.

\begin{figure}[h]
\centering
\includegraphics[width=0.55\textwidth]{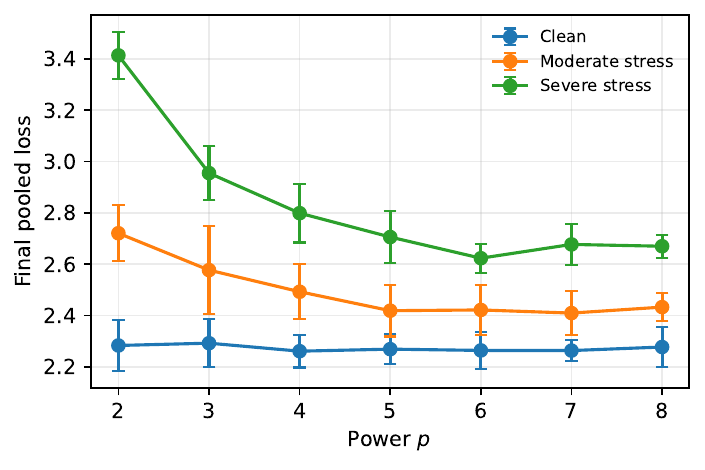}
\caption{Mean final pooled loss versus $p$; bars show one sample standard deviation over five paired seeds.}
\label{fig:loss-power}
\end{figure}

\subsection{Displacement compression and local cost}
\label{subsec:displacement-cost}

The displacement-tail statistic decreases monotonically with $p$ in every condition; see Table~\ref{tab:diagnostics} and Fig.~\ref{fig:tail-power}. Under severe stress, it falls from $4.936$ at $p=2$ to $2.086$ at $p=8$, a reduction of about $57.7\%$. The stress-test-to-unmodified mean-displacement ratio likewise decreases from $5.21$ to $1.91$. Thus, increasing $p$ continues to equalize relative displacement magnitudes even after predictive loss has stopped improving.

\begin{table*}[h]
\caption{Mean displacement-tail statistic and severe-stress Armijo trial-point cost. The final column is relative to $p=2$.}
\label{tab:diagnostics}
\centering
\footnotesize
\begin{tabular}{cccccc}
\toprule
$p$ & Clean tail & Moderate tail & Severe tail & Severe trials & Trial increase \\
\midrule
$2$ & $1.448$ & $3.070$ & $4.936$ & $7835.0$ & $0.00\%$ \\
$3$ & $1.428$ & $2.561$ & $3.506$ & $8472.2$ & $8.13\%$ \\
$4$ & $1.395$ & $2.210$ & $2.763$ & $8810.8$ & $12.45\%$ \\
$5$ & $1.362$ & $2.015$ & $2.437$ & $9214.6$ & $17.61\%$ \\
$6$ & $1.343$ & $1.905$ & $2.284$ & $9553.6$ & $21.93\%$ \\
$7$ & $1.327$ & $1.817$ & $2.172$ & $9764.6$ & $24.63\%$ \\
$8$ & $1.302$ & $1.777$ & $2.086$ & $10069.6$ & $28.52\%$ \\
\bottomrule
\end{tabular}
\end{table*}

\begin{figure}[h]
\centering
\includegraphics[width=0.55\textwidth]{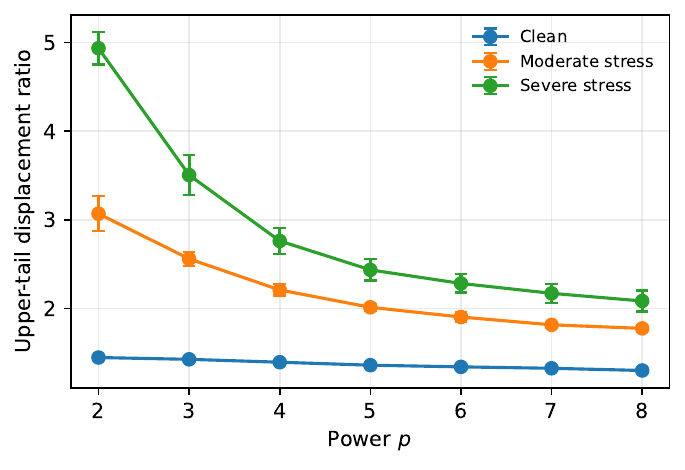}
\caption{Mean upper-tail displacement ratio versus $p$; bars show one sample standard deviation.}
\label{fig:tail-power}
\end{figure}

\begin{figure*}[h]
\centering
\begin{minipage}{0.4\textwidth}
\centering
\includegraphics[width=\linewidth]{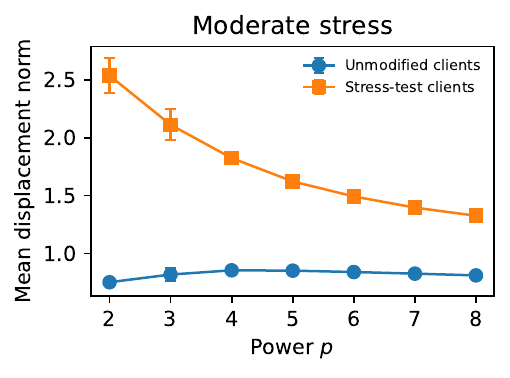}
\end{minipage}\hfill
\begin{minipage}{0.4\textwidth}
\centering
\includegraphics[width=\linewidth]{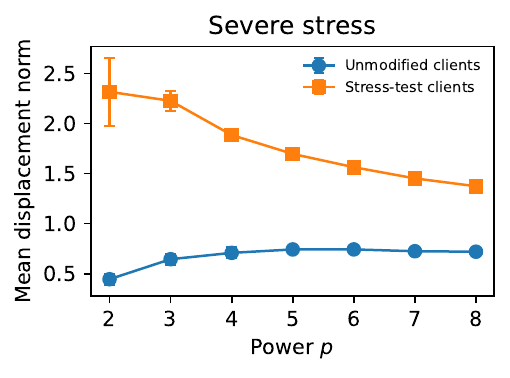}
\end{minipage}
\caption{Mean client-displacement norms versus $p$ under moderate and severe stress. Bars show one sample standard deviation over five seeds. Increasing $p$ first allows more movement in the unmodified group while progressively reducing stress-test-client movement. At the largest powers, the unmodified-client mean also turns downward, consistent with the fixed-radius affine limit.}
\label{fig:group-power}
\end{figure*}

Figure~\ref{fig:group-power} shows how the tail reduction is produced. Relative to $p=2$, intermediate powers increase the mean displacement of unmodified clients while decreasing that of stress-test clients, so the group separation narrows from both directions. For larger powers, the unmodified-client mean also declines. This turnover supports a selective-response interpretation rather than uniform shrinkage and helps explain why the smallest tail at $p=8$ does not result in the best predictive loss.

All powers use the same number of local iterations, backward passes, and client--server messages. Their measured computational difference is the number of Armijo trial-point forward evaluations. Under severe stress, this count grows from $7835.0$ at $p=2$ to $9553.6$ at $p=6$, $9764.6$ at $p=7$, and $10069.6$ at $p=8$. No run records a complete line-search failure or zero-step fallback. The writer-level tenth-percentile accuracy improves sharply from $2.35\%$ at $p=2$ to between $18.49\%$ and $20.15\%$ for $p=5$--$8$ under severe stress, but it does not improve monotonically with the exponent.

The numerical pattern therefore matches the theoretical tradeoff. Increasing $p$ compresses relative displacement disparities, but beyond an intermediate range the additional compression no longer improves the prediction and continues to increase local search cost.

%%%%%%%%%%%%%%%%%%%%%%%%%%%%%%%%%%%%%%%%%%%%%%%%%%%%
\section{Discussion: Choosing the Exponent}
\label{sec:discussion}

We would like to emphasize that there is no universally optimal exponent. The parameter $p$ controls how sharply the regularizer changes around $R$: larger powers are flatter below the crossover and steeper above it. In the affine model, they drive all nonzero displacement magnitudes toward $R$, while the local curvature outside the crossover grows rapidly. Consequently, choosing $p$ balances three effects: suppression of extreme movements, preservation of informative magnitude differences, and local backtracking cost.

For the present FEMNIST protocol, the useful range is approximately $p=5$--$7$, and no single exponent dominates every criterion. The choice $p=5$ is the conservative option because it has the lowest local cost within this range and gives the highest moderate-stress accuracy. The choice $p=6$ is preferable when severe large-displacement robustness is the main priority, since it gives the lowest severe-stress loss and highest severe-stress accuracy. The choice $p=7$ gives the lowest moderate-stress loss and the best clean-plus-moderate score, although its advantage over $p=5$ and $p=6$ is small. The result at $p=8$ rules out the simplistic conclusion that a larger exponent is always better: it gives the smallest displacement tail but does not improve predictive performance and requires the most trial evaluations. Table~\ref{tab:p-guide} summarizes the observed tradeoffs in this protocol.

In practice, $p$ should be selected jointly with $\mu$ and $R$ by validation. A reasonable rule is to choose the smallest exponent whose validation loss lies within a prescribed tolerance of the best observed value, thereby avoiding unnecessary curvature and overcompression. The current comparison deliberately keeps $(\mu,R)$ fixed to isolate the exponent, so it does not establish the jointly optimal triplet $(p,\mu,R)$. A stronger tuning study should use several calibration seeds and keep the five evaluation seeds untouched.

The scope remains limited to one 60-writer subset, one architecture, five seeds, and a composite stress intervention that changes labels, local steps, and initial step sizes together. The study does not test unseen writers, isolate individual stress components, or explore fractional powers. Finally, the affine calculation and smooth fixed-batch safeguard do not prove convergence of the implemented changing-minibatch ReLU procedure with partial participation.

\begin{table}[h]
\caption{Empirical exponent choices in the present protocol.}
\label{tab:p-guide}
\centering
%\footnotesize
\begin{tabular}{ll}
\toprule
Priority & Observed choice \\
\midrule
Lowest cost within the plateau & $p=5$ \\
Best severe-stress performance & $p=6$ \\
Best moderate-stress loss & $p=7$ \\
Strongest tail compression & $p=8$ \\
\bottomrule
\end{tabular}
\end{table}

%%%%%%%%%%%%%%%%%%%%%%%%%%%%%%%%%%%%%%%%%%%%%%%%%%%%%%%%%%%%%%%%%%%%%%%%%%%%%%%%%%%%
%%%%%%%%%%%%%%%%%%%%%%%%%%%%%%%%%%%%%%%%%%%%%%%%%%%%%%%%%%%%%%%%%%%%%%%%%%%%%%%%%%%%
\section{Conclusion} \label{sec:conclusion}
We introduced HiFedProx, a first-order federated procedure based on a scale-matched family of power-type client regularizers. The exponent $p$ controls the radial shape without changing the regularization force at the reference displacement $R$. The affine response shows that larger powers compress relative displacement disparities and approach a fixed-radius behavior, while the curvature analysis predicts increasing local search difficulty outside the crossover region.

The paired FEMNIST study over $p\in\{2,3,4,5,6,7,8\}$ confirms this tradeoff. Predictive performance improves from the quadratic model to an intermediate high-order range: $p=7$ provides us with the lowest moderate-stress loss, $p=5$ gives the highest moderate-stress accuracy, and $p=6$ results in the best severe-stress loss and accuracy. Our biggest choice of $p=8$ further reduces displacement tails but does not improve predictive performance and increases Armijo trial cost. Thus, the exponent should be calibrated rather than maximized. In the present experiment, $p=5$--$7$ provides the most useful range.

%%%%%%%%%%%%%%%%%%%%%%%%%%%%%%%%%%%%%%%%%%%%%%%%%%%%%%%%%%%%%%%%%%%%%%%%%%%%%%%%%%%%
%%%%%%%%%%%%%%%%%%%%%%%%%%%%%%%%%%%%%%%%%%%%%%%%%%%%%%%%%%%%%%%%%%%%%%%%%%%%%%%%%%%%


\begin{thebibliography}{00}
\bibitem{mcmahan2017fedavg} H. B. McMahan, E. Moore, D. Ramage, S. Hampson, and B. Ag\"uera y Arcas, ``Communication-efficient learning of deep networks from decentralized data,'' in \emph{Proc. AISTATS}, 2017, pp. 1273--1282.
    
\bibitem{li2020fedprox} T. Li, A. K. Sahu, M. Zaheer, M. Sanjabi, A. Talwalkar, and V. Smith, ``Federated optimization in heterogeneous networks,'' in \emph{Proc. MLSys}, 2020, pp. 429--450.

\bibitem{nesterov2021inexact}
Y. Nesterov, ``Inexact high-order proximal-point methods with auxiliary search procedure,''
\emph{SIAM J. Optim}, vol. 31, pp. 2807--2828, 2021.

\bibitem{Nesterov2023a}
Y. Nesterov, ``Inexact accelerated high-order proximal-point Methods,''
 \emph{Math. Program.}, vol. 197, pp. 1--26, 2023.
 

\bibitem{Ahookhosh24}
M. Ahookhosh and Y. Nesterov, ``High-order methods beyond the classical complexity bounds: inexact high-order proximal-point methods,'' \emph{Math. Prog.}, vol. 208, pp. 365--407, 2024.
 
\bibitem{Kabganidiff}
A. Kabgani and M. Ahookhosh, 
``Moreau envelope and proximal-point methods under the lens of
high-order regularization,''
\emph{Set-Valued Var. Anal.},
vol. 33, no. 4, Art. no. 47, pp. 1--35, 2025.

\bibitem{KabganiItsDEAL}
A. Kabgani and M. Ahookhosh, ``ItsDEAL: Inexact two-level smoothing descent algorithms for weakly convex optimization,'' arXiv:2501.02155, 2025.

\bibitem{Kabgani24itsopt}
A. Kabgani and M. Ahookhosh,
``ItsOPT: An inexact two-level smoothing framework for nonconvex
optimization via high-order Moreau envelope,''
\emph{SIAM J. Optim.}, in press.

\bibitem{Ahookhosh2025}
M. Ahookhosh, A. Iusem, A. Kabgani, and F. Lara,
``Asymptotic convergence analysis of high-order proximal-point methods
beyond sublinear rates,''
\emph{SIAM J. Optim.}, in press.

\bibitem{Kabgani2026Robust}
A. Kabgani, F. Lara and M. Ahookhosh, ``Robust learning meets quasar-convex optimization: Inexact high-order proximal-point methods,'' 	arXiv:2605.08474, 2026

\bibitem{dinh2020pfedme} C. T. Dinh, N. H. Tran, and T. D. Nguyen, ``Personalized federated learning with Moreau envelopes,'' in \emph{Proc. NeurIPS}, 2020.

\bibitem{hajarizadeh2026dynamu} A. Hajarizadeh and S. Gupta, ``DynaMu: Loss-guided adaptive proximal regularization for federated learning,'' in \emph{Proc. IEEE/CVF CVPR Workshops}, 2026, pp. 3321--3329.

\bibitem{karimireddy2020scaffold} S. P. Karimireddy \emph{et al.}, ``SCAFFOLD: Stochastic controlled averaging for federated learning,'' in \emph{Proc. ICML}, 2020, pp. 5132--5143.
    
\bibitem{wang2020fednova} J. Wang, Q. Liu, H. Liang, G. Joshi, and H. V. Poor, ``Tackling the objective inconsistency problem in heterogeneous federated optimization,'' in \emph{Proc. NeurIPS}, 2020.
    
\bibitem{jhunjhunwala2022fedvarp} D. Jhunjhunwala, P. Sharma, A. Nagarkatti, and G. Joshi, ``FedVARP: Tackling the variance due to partial client participation in federated learning,'' in \emph{Proc. UAI}, 2022, pp. 906--916.

\bibitem{shi2023fedl1} Y. Shi, Y. Zhang, P. Zhang, Y. Xiao, and L. Niu, ``Federated learning with $\ell_1$ regularization,'' \emph{Pattern Recognit. Lett.}, vol. 172, pp. 15--21, 2023.

\bibitem{safaryan2022fednl} M. Safaryan, R. Islamov, X. Qian, and P. Richt\'arik, ``FedNL: Making Newton-type methods applicable to federated learning,'' in \emph{Proc. ICML}, 2022, pp. 18959--19010.

\bibitem{zhang2022clipping} X. Zhang, X. Chen, M. Hong, S. Wu, and J. Yi, ``Understanding clipping for federated learning: Convergence and client-level differential privacy,'' in \emph{Proc. ICML}, 2022, pp. 26048--26067.

\bibitem{armijo1966minimization} L. Armijo, ``Minimization of functions having Lipschitz continuous first partial derivatives,'' \emph{Pacific J. Math.}, vol. 16, no. 1, pp. 1--3, 1966.
    
\bibitem{caldas2018leaf} S. Caldas \emph{et al.}, ``LEAF: A benchmark for federated settings,'' arXiv:1812.01097, 2018.

\bibitem{pytorchAutograd} PyTorch Contributors, ``Autograd mechanics,'' \emph{PyTorch Documentation}, accessed Jul. 2026. [Online]. Available: \url{https://docs.pytorch.org/docs/stable/notes/autograd.html}

\bibitem{efron1979bootstrap} B. Efron, ``Bootstrap methods: Another look at the jackknife,'' \emph{Ann. Statist.}, vol. 7, no. 1, pp. 1--26, 1979.

\bibitem{holm1979multiple} S. Holm, ``A simple sequentially rejective multiple test procedure,'' \emph{Scand. J. Statist.}, vol. 6, no. 2, pp. 65--70, 1979.

% \bibitem{pytorchReproducibility} PyTorch Contributors, ``Reproducibility,'' PyTorch Documentation, accessed Jul. 2026.


\end{thebibliography}
\end{document}